\documentclass[11pt, a4paper]{article}

\usepackage{lmodern}
\usepackage{microtype}

\usepackage{thmtools, thm-restate}

\usepackage[utf8]{inputenc} 
\usepackage[T1]{fontenc} 
\usepackage{geometry} 
\usepackage[%
    colorlinks      = true,
  linktocpage     = true,
  linkcolor       = blue,
  citecolor       = blue,
  urlcolor        = blue,
  breaklinks =true,
  hypertexnames=false,
]{hyperref}

\usepackage{setspace,placeins}
\usepackage[affil-it]{authblk}

\usepackage{amsmath, amsthm, amsfonts, amssymb, amsbsy, bigstrut, graphicx, enumerate,  upref, longtable, comment, booktabs, array, caption, subcaption}

\usepackage[numbers, sort&compress]{natbib}

\newcommand{\ep}{\epsilon}

\newtheorem{thm}{Theorem}[section]
\newtheorem{lmm}[thm]{Lemma}
\newtheorem{cor}[thm]{Corollary}

\theoremstyle{definition}

\newcommand{\ma}{\mathcal{A}}

\newcommand{\pp}{\mathbb{P}}

\newcommand{\rr}{\mathbb{R}}

\newcommand{\tr}{\operatorname{Tr}}

\newcommand{\mb}{\mathcal{B}}

\numberwithin{equation}{section}

\usepackage[utf8]{inputenc}
\usepackage{amsmath}
\usepackage{amsfonts}
\usepackage{mathtools}
\usepackage{tikz}
\usetikzlibrary{decorations.pathreplacing}
\usepackage{amsthm}
\usepackage[english]{babel}
\usepackage{fancyhdr}
\usepackage{amssymb}
\usepackage{enumitem}
\usepackage{mathrsfs}

\newtheorem{rmk}[thm]{Remark}

\renewcommand{\tilde}{\widetilde}

\begin{document}
\title{Convergence of gradient descent for deep neural networks}
\date{September 21, 2026}

\author{Sourav Chatterjee\thanks{Department of Statistics, Stanford University, USA. Email: \href{mailto:souravc@stanford.edu}{\tt souravc@stanford.edu}. 
}}
\affil{Stanford University}


\maketitle

\begin{abstract}
We give a simple \emph{local Polyak--{\L}ojasiewicz} (PL) criterion that guarantees linear (exponential) convergence of gradient flow and gradient descent to a zero-loss solution of a nonnegative objective.
We then verify this criterion for the squared training loss of a feedforward neural network with smooth, strictly increasing activation functions, in a regime that is complementary to the usual over-parameterized analyses: the network width and depth are fixed, while the input data vectors are assumed to be linearly independent (in particular, the ambient input dimension is at least the number of data points). A notable feature of the verification is that it is \emph{constructive}: it leads to a simple ``positive'' initialization (zero first-layer weights, strictly positive hidden-layer weights, and sufficiently large output-layer weights) under which gradient descent provably converges to an interpolating global minimizer of the training loss.
We also discuss a probabilistic corollary for random initializations, clarify its dependence on the probability of the required initialization event, and provide numerical experiments showing that this theory-guided initialization can substantially accelerate optimization relative to standard random initializations at the same width.
\end{abstract}

\section{A convergence criterion for gradient descent}\label{introsec}
\subsection{Background}
The goal of gradient descent is to find a minimum of a differentiable function $f:\rr^p\to \rr$ by starting at some $x_0\in \rr^p$, and then iteratively moving in the direction of steepest descent, as
\[
x_{k+1} = x_k -\eta_k \nabla f(x_k),
\]
where $\nabla f$ is the gradient of $f$, and the step sizes $\eta_k$ may remain fixed or vary with $k$. This is the discretization of the continuous-time gradient flow $\phi:[0,\infty)\to \rr^p$, which is the solution of the differential equation
\[
\frac{d}{dt}\phi(t) = -\nabla f(\phi(t))
\]
with $\phi(0)=x_0$. Gradient descent and its many variants are indispensable tools in all branches of science and engineering, and particularly in modern machine learning and data science. The convergence properties of gradient descent are well-understood when the objective function $f$ is convex~\cite{nesterov04, boydvandenberghe04}, and in general nonconvex optimization, even \emph{certifying} local optimality can be computationally intractable (NP-hard, and related decision problems are NP-complete); see, for example,~\cite{murtykabadi87}.  Of course, for many ``tame'' objectives gradient methods do converge to critical points; the point here is that worst-case guarantees for locating (or certifying) good minima of arbitrary nonconvex functions are necessarily limited.
In spite of this, gradient descent is widely used in practice to find local and global minima in highly nonconvex problems, especially in high dimensions. For example, it has been observed that gradient descent can often find global minima of training loss in deep learning~\cite{zhangetal21, goodfellowetal14}, which is one of the reasons behind great success of the deep learning revolution~\cite{lecunetal15,goodfellowetal14}.

Beyond worst-case guarantees, one motivation for our results is to obtain convergence theory that is not merely descriptive but \emph{prescriptive}: it can suggest concrete algorithmic choices.
In the neural-network application, the sufficient conditions can be verified by analyzing a local lower bound on the empirical tangent kernel, and this yields a particularly simple, sign-constrained starting point.
In particular, the ``positive'' initialization singled out by the proof (zero first-layer weights, positive hidden-layer weights, and a suitably scaled output layer) is easy to implement.
Our numerical experiments show that this theory-guided start can dramatically speed up training compared to standard random initializations at the same (small) width, serving as a proof-of-principle that optimization theory can be used to design effective initializations.

This article presents a novel criterion for convergence of gradient descent to a global minimum. The criterion is related to (and maybe seen as a strengthening of) the classical Kurdyka--{\L}ojasiewicz inequality~\cite{attouchetal10, lojasiewicz63, kurdyka98}. It is also related to results from nonsmooth analysis, such as those in \cite{ioffe00, drusvyatskiyetal15, bolteetal10, corvellecmotreanu08}. Our proof idea bears close resemblance with the classical `{\L}ojasiewicz trapping argument'~\cite{absiletal05, kurdyka98}. Nevertheless, the convergence criterion stated below has not appeared in this exact form previously in the literature.

\subsection{Main results}
 Let $p$ be a positive integer. Let $f: \rr^p\to [0,\infty)$ be a nonnegative $C^2$ function. Take any $x_0\in \rr^p$ and $r>0$. Let $B(x_0,r)$ denote the closed Euclidean ball of radius $r$ centered at $x_0$. Define
\[
\alpha(x_0,r) := \inf_{x\in B(x_0,r), \, f(x)\ne 0} \frac{|\nabla f(x)|^2}{f(x)},
\]
where  $|\nabla f(x)|$ is the Euclidean norm of $\nabla f(x)$. If $f(x)=0$ for all $x\in B(x_0,r)$, then we let $\alpha(x_0,r):=\infty$. Our main assumption is that for some $r>0$, 
\begin{align}\label{mainassump}
4f(x_0) < r^2\alpha(x_0,r). 
\end{align}
The quantity $\alpha(x_0,r)$ in~\eqref{mainassump} is a \emph{local PL constant}: on the ball $B(x_0,r)$, the inequality $|\nabla f(x)|^2 \ge \alpha(x_0,r) f(x)$ is exactly a local version of the Polyak--{\L}ojasiewicz (PL) or \emph{gradient-dominance} condition (compare, e.g.,~\cite{lojasiewicz63, polyak63, kariminutini16}). In particular, it implies a Kurdyka--{\L}ojasiewicz inequality on the same ball with exponent $\frac{1}{2}$, but it is \emph{stronger} than a generic KL inequality because it controls the objective value linearly.

Under the assumption \eqref{mainassump}, we have two results. The first result, stated below, shows that the gradient flow started at $x_0$  converges exponentially fast to a global minimum of $f$ in $B(x_0,r)$ where $f$ is zero. The existence of a global minimum in $B(x_0,r)$ is a part of the conclusion, and not an assumption. Since $f$ is assumed to be nonnegative, any point $x^*$ with $f(x^*)=0$ achieves the minimum value of $f$ and is therefore a (possibly non-unique) \emph{global minimizer} of $f$; our theorems guarantee that such a point lies within the ball and is reached from $x_0$. The proof is in \textsection\ref{thm2proof}.
\begin{thm}\label{mainthm2}
Let $f$, $x_0$, and $\alpha$ be as above. Assume that \eqref{mainassump} holds for some $r>0$, and let $\alpha := \alpha(x_0,r)$. If $f(x_0)=0$, the flow is stationary at $x_0$. Otherwise, there is a unique solution of the gradient flow equation
\[
\frac{d}{dt} \phi(t) = -\nabla f(\phi(t))
\]
on $[0,\infty)$ with $\phi(0)=x_0$, and this flow stays in $B(x_0,r)$ for all time, and converges to some $x^*\in B(x_0,r)$ where $f(x^*)=0$. Moreover, for each $t\ge 0$, we have
\[
|\phi(t) - x^*|\le re^{-\alpha t/2} \ \ \text{ and } \ \  f(\phi(t))\le e^{-\alpha t} f(x_0). 
\]
\end{thm} 
Our second result is the analogue of Theorem \ref{mainthm2} for gradient descent. It says that under  the condition \eqref{mainassump}, gradient descent started at $x_0$, with a small enough step size, converges to a global minimum of $f$ in $B(x_0,r)$. Again, the existence of a global minimum in $B(x_0,r)$ is a part of the conclusion. This result is proved in \textsection\ref{thm1proof}.
\begin{thm}\label{mainthm1}
Let $f$, $x_0$, and $\alpha$ be as above. Assume that~\eqref{mainassump} holds for some $r>0$, and let $\alpha := \alpha(x_0,r)$. If $f(x_0)=0$, gradient descent is stationary at $x_0$ for every $\eta>0$. Otherwise, define
\[
\rho := \sqrt{\frac{4f(x_0)}{r^2\alpha}}\,,
\]
so that $\rho\in (0,1)$ by~\eqref{mainassump}. Fix any $\ep$ satisfying $0<\ep<1-\rho$, and define
\[
L_1 := \sup_{x\in B(x_0,r)} \max_{1\le i\le p} \bigl|\partial_i f(x)\bigr|
\qquad\text{and}\qquad
L_2 := \sup_{x\in B(x_0,2r)} \max_{1\le i,j\le p} \bigl|\partial_{i} \partial_jf(x)\bigr|.
\]
Choose any step size $\eta>0$ such that
\[
\eta \le \min \biggl\{\frac{r}{L_1\sqrt{p}}, \, \frac{2\ep}{L_2p}\biggr\},
\]
and iteratively define
\[
x_{k+1} := x_k - \eta \nabla f(x_k)
\]
for each $k\ge 0$. Then $x_k\in B(x_0,r)$ for all $k$, and as $k\to\infty$, $x_k$ converges to a point $x^*\in B(x_0,r)$ where $f(x^*)=0$. Moreover, with
\[
\kappa := \min\{1, (1-\ep)\alpha \eta\},
\]
we have for each $k\ge 0$ that
\[
|x_k-x^*|\le (1-\kappa)^{k/2} r \ \ \text{ and } \ \ f(x_k)\le (1-\kappa)^kf(x_0). 
\]
\end{thm}
Note that the above results have nothing to say about variants of gradient descent, such as stochastic gradient descent. Adding a stochastic component to the gradient descent algorithm has various benefits, such as helping it escape saddle points~\cite{jinetal17}. Since it is known that stochastic gradient methods often asymptotically follow the path of a differential equation~\cite{duchiruan18}, it would be interesting to see if analogues of Theorems~\ref{mainthm2} and~\ref{mainthm1} can be proved for stochastic gradient descent. We also do not have anything to say about algorithms that aim to find critical points instead of global minima in nonconvex problems, such as the ones surveyed in \cite{danilovaetal20, carmonetal18}. For a  survey of the many variants of stochastic gradient descent and their applications in machine learning, see~\cite{netrapalli19}. For a comprehensive account of all variants of gradient descent, see~\cite{ruder16}. For some essential limitations of nonconvex optimization, see~\cite{arjevanietal19}.

\begin{rmk}[Beyond $C^2$ smoothness]
Theorems~\ref{mainthm2} and~\ref{mainthm1} are stated for $C^2$ objectives because our discrete-time analysis uses a second-order Taylor expansion and uniform bounds on the Hessian on a neighborhood of the iterates.
For nonsmooth problems, one typically replaces the gradient by a suitable generalized gradient (for example, the limiting subdifferential) and works with a \emph{descent lemma} for a specific algorithm (proximal gradient, forward--backward splitting, alternating minimization, etc.).
Many nonsmooth objectives of interest in optimization are \emph{weakly convex} or \emph{prox-regular}, and many such functions satisfy a Kurdyka--{\L}ojasiewicz (KL) inequality; see, e.g.,~\cite{attouchetal10, bolteetal10, drusvyatskiyetal15}.
A continuous-time analogue of Theorem~\ref{mainthm2} can often be obtained in that setting by combining a KL inequality with a trapping argument.
Extending Theorem~\ref{mainthm1} in a comparably clean way to nonsmooth objectives would require additional machinery and is beyond the scope of this paper.
\end{rmk}

\subsection{How the criterion will be used for neural networks}\label{subsec:roadmap}
The application in \textsection\ref{deepsec} proceeds by verifying~\eqref{mainassump} for the squared training loss $S(w)$ of a neural network.
A key observation is that for a least-squares objective of the form
\[
S(w)=\frac1n\sum_{i=1}^n (y_i-f(x_i,w))^2,
\]
the ratio appearing in the definition of $\alpha(\cdot,\cdot)$ can be written in terms of the \emph{empirical tangent kernel} (or empirical NTK)
\[
H(w) := (\nabla f(x_i,w)^T \nabla f(x_j,w))_{1\le i,j\le n},
\]
where $u^T$ denotes the transpose of a column vector $u$. 
Indeed, letting $z(w):=(y_1-f(x_1,w),\ldots,y_n-f(x_n,w))^T \in\rr^n$, one checks that
\[
|\nabla S(w)|^2 = \frac{4}{n^2} z(w)^T H(w) z(w)
\quad\text{and}\quad
S(w)=\frac1n|z(w)|^2,
\]
so that
\[
\frac{|\nabla S(w)|^2}{S(w)} = \frac{4}{n}\frac{z(w)^T H(w) z(w)}{|z(w)|^2}
\ge \frac{4}{n}\lambda_{\min}(H(w)),
\]
where $\lambda_{\min}(M)$ denotes the minimum eigenvalue of a Hermitian matrix $M$. 
Consequently, a sufficient route to~\eqref{mainassump} is to (i) find a lower bound for $\lambda_{\min}(H(w))$ on a ball around the initialization $w_0$, and (ii) choose the initialization so that $S(w_0)$ is small relative to that lower bound.
In \textsection\ref{deepsec}, we obtain such a lower bound in a high-dimensional data regime by controlling the first-layer gradients and using the positivity of certain products of weights.

\section{Application to deep neural networks}\label{deepsec} 
\subsection{Feedforward neural networks}
A feedforward neural network has the following components:
\begin{itemize}
\item A positive integer $L$, which denotes the number of layers. It is sometimes called the depth of the network. The number $L-1$ denotes the `number of hidden layers'. To avoid trivialities, we will assume that $L\ge 2$; i.e., there is at least one hidden layer. 
\item A positive integer $d$, called the `input dimension'. The input data take value in $\rr^d$.
\item A sequence of positive integers $d_1,\ldots,d_L$, denoting the dimensions of layers $1,\ldots,L$. The maximum of $d_1,\ldots,d_L$ is called the width of the network. The dimension of the `output layer', $d_L$, is often taken to be $1$. We will henceforth take $d_L = 1$. 
\item A sequence of `weight matrices' $W_1,\ldots, W_L$ with real entries, where $W_\ell$ has dimensions $d_\ell\times d_{\ell-1}$ (with $d_0:=d$).
\item A sequence of `bias vectors' $b_1,\ldots, b_L$, where $b_\ell\in \rr^{d_\ell}$. 
\item A sequence of `activation functions' $\sigma_1,\ldots,\sigma_L :\rr \to \rr$.  Usually, the activation function for the output layer, $\sigma_L$, is taken to be the identity map, i.e., $\sigma_L(x)=x$.  We will henceforth take $\sigma_L$ to be the identity map. 
\end{itemize}
The feedforward neural network with the above components is defined as follows. Denote $w := (W_1,b_1,\ldots,W_L,b_L)$, viewing it as a vector in $\rr^p$, where 
\begin{align}\label{pdef}
p := \sum_{\ell=1}^L d_\ell (d_{\ell-1} + 1). 
\end{align}
Given a value of $w$, the neural network defines a map $f(\cdot,w)$ from  the input space $\rr^d$ into the output space $\rr$, as
\[
f(x,w) := \sigma_L(W_L\sigma_{L-1}(\cdots W_2 \sigma_1(W_1 x+b_1) + b_2\cdots)+b_L),
\]
where the activation functions $\sigma_1,\ldots ,\sigma_L$ act componentwise on vectors of dimensions $d_1,\ldots,d_L$.

Suppose that we have a neural network as above, and some input data $x_1,\ldots,x_n\in \rr^d$ and output data $y_1,\ldots,y_n\in \rr$. Suppose that we want to `fit' the model to the data by trying to minimize the squared error loss
\begin{align}\label{swdef}
S(w) := \frac{1}{n}\sum_{i=1}^n (y_i - f(x_i,w))^2
\end{align}
We study full-batch gradient flow and gradient descent on the training objective~\eqref{swdef}, namely
\[
\frac{d}{dt}w(t) = -\nabla S(w(t))
\qquad\text{and}\qquad
w_{k+1} = w_k - \eta \nabla S(w_k),
\]
under smoothness assumptions on the activation functions that ensure $S$ is $C^2$.

\subsection{A high-dimensional data regime}\label{subsec:highdim}
Our main convergence theorems below assume that the inputs $x_1,\ldots,x_n$ are linearly independent.
In particular, this requires that the ambient input dimension satisfies $d\ge n$.
This is a \emph{different} regime from the by-now standard ``over-parameterized'' analyses of neural-network training, where one typically keeps $d$ fixed and takes the \emph{width} (and hence the number of parameters) to scale with $n$.

The $d\ge n$ assumption can be natural in high-dimensional statistics (small $n$, large $d$), or after applying a high-dimensional feature map/embedding to the inputs.
It is also worth emphasizing what this paper does \emph{not} claim: we do not provide a guarantee for generic modern deep-learning practice in the regime $n\gg d$.
Rather, our contribution is to show that in the complementary high-dimensional setting, \emph{width need not grow with $n$} in order for gradient descent to provably reach an interpolating solution.

In fact, linear independence of the inputs already implies strong interpolation capacity for very small networks.
For example, when $L=2$ and $d_1=1$, the map $x\mapsto f(x,w)$ is of the form $x\mapsto a \sigma_1(u^T x)$.
In the centered-activation setting of Theorem~\ref{deepthm}, the range of $\sigma_1$ contains an interval about zero. Choose $a>0$ large enough that each $y_i/a$ belongs to this interval, and choose $t_i$ with $\sigma_1(t_i)=y_i/a$. Since $X:=[x_1\; x_2 \;\cdots \; x_n]$ has full column rank, one can solve $u^Tx_i=t_i$ simultaneously, yielding exact interpolation.
Our results show that, for \emph{any} fixed depth $L$ and fixed hidden-layer dimensions $d_1,\ldots,d_{L-1}$, gradient descent can be initialized so as to find such an interpolating solution, and does so with a linear rate.


\subsection{Our results for feedforward neural networks}\label{subsec:related}
There is an enormous body of work on optimization for neural networks.
We only summarize the strands that are the most relevant for the results stated below, and refer to the survey~\cite{sun19} (and its references) for broader background. For convex and linear models, convergence results for gradient-type methods are classical; for example, convex neural networks were analyzed in~\cite{bachmoulines13, bach17, moulinesbach11}, and deep linear networks in~\cite{aroraetal18, bartlettetal18, jitelgarsky18}.
For general (nonconvex, nonlinear) networks, many recent global-convergence guarantees are proved in an over-parameterized regime where the hidden layers are very wide and the empirical tangent kernel $H(w)$ is well conditioned and remains stable along training.
Representative results include~\cite{duetal18, allenzhuetal19b, liliang18, wuetal18, eetal20, soltanolkotabietal18, hardtma16, duetal19, jacotetal18, zouetal20, caogu19, caogu20, sankararamanetal20, liuetal22}; see also the discussion in~\cite{jentzenriekert21}.
There are also works pushing beyond the ``infinite-width'' or NTK regime and obtaining convergence under milder (but still width-dependent) conditions, for example~\cite{jentzenriekert21, robinetal22}.

Our theorems below are complementary to the results cited above: we do not assume width scales with $n$, and our lower bound on $\lambda_{\min}(H(w))$ comes instead from the high-dimensional structure of the data (linear independence) together with an explicit initialization that prevents sign cancellations in the relevant products of weights.
Technically, the proofs still proceed by verifying a local PL inequality for $S$ via a lower bound on $\lambda_{\min}(H(w))$, but the mechanism is different from the typical over-parameterized stability arguments.

Finally, because our analysis is based on a PL/KL-type inequality, it is related in spirit to a classical line of work that uses (Kurdyka--) {\L}ojasiewicz gradient inequalities to study convergence of gradient-like dynamics; see, e.g.,~\cite{lojasiewicz63, absiletal05, attouchetal10, bolteetal10, drusvyatskiyetal15}.
More recently, {\L}ojasiewicz-type inequalities have also been used to connect optimization dynamics with generalization in learning problems; see, for example,~\cite{liuetal22}.

Our first result about feedforward neural networks is the following. The proof is in \textsection\ref{deepproof}.
\begin{thm}\label{deepthm}
Consider a feedforward neural network with depth $L\ge 2$, output dimension $d_L=1$, and $\sigma_L=$ identity.
Assume that the activation functions $\sigma_1,\ldots,\sigma_{L-1}$ are twice continuously differentiable, satisfy $\sigma_\ell(0)=0$, and are strictly increasing in the sense that $\sigma_\ell'(x)>0$ for all $x$ and $\ell$. Given input data $x_1,\ldots,x_n\in \rr^d$ and output data $y_1,\ldots,y_n\in \rr$, define the squared loss $S$ as in~\eqref{swdef}.
Assume that the input vectors $x_1,\ldots,x_n$ are linearly independent, and let
\[
\lambda_X := \lambda_{\min}\biggl(\frac{1}{n}X^T X\biggr) >0,
\]
where $X$ is the $d\times n$ matrix whose $i^{\textup{th}}$ column is $x_i$. Let $w_0=(W_1,b_1,\ldots,W_L,b_L)$ be an initialization such that $b_1=\cdots=b_L=0$, $W_1=0$, and every entry of each hidden-layer weight matrix $W_2,\ldots,W_{L-1}$ is strictly positive.
For $L\ge 3$, define
\[
\delta := \min_{2\le \ell\le L-1}\ \min_{i,j} (W_\ell)_{ij}
\quad\text{and}\quad
K := \max_{2\le \ell\le L-1}\ \max_{i,j} (W_\ell)_{ij}.
\]
When $L=2$, set $\delta=K=1$ by convention. Write
\[
R_X:=\max_{1\le i\le n}|x_i|,\qquad Y_2:=\frac1n\sum_{i=1}^n y_i^2.
\]
There is an explicit constant $A_0>0$ (see Remark~\ref{rmk:constants}) depending only on
\[
\lambda_X,\ R_X,\ \delta,\ K,\ Y_2,\ L,\ d_1,\ldots,d_{L-1},\quad
(\sigma_\ell)_{\ell=1}^{L-1},
\]
such that, for every choice of $W_L$ whose entries are at least $A_0$:
\begin{itemize}
\item the gradient flow $\dot w(t)=-\nabla S(w(t))$ with $w(0)=w_0$ converges to a point $w^*$ with $S(w^*)=0$; and
\item there is a step-size bound $\eta_0(w_0)>0$ such that gradient descent $w_{k+1}=w_k-\eta\nabla S(w_k)$ converges to a point $w^*$ with $S(w^*)=0$ whenever $0<\eta\le\eta_0(w_0)$.
\end{itemize}
The bound $\eta_0(w_0)$ may depend on the full data, the architecture, the activation functions, and the actual initialization $w_0$, including the entries of $W_L$. No uniform positive step-size bound over all unbounded choices of $W_L$ is asserted.

\end{thm}

\begin{rmk}[Width versus ambient dimension]\label{rmk:width}
Recall that the width of the network is $m := \max\{d_1,\ldots,d_L\}$ and does \emph{not} include the input dimension $d_0=d$.
Theorem~\ref{deepthm} works in the high-dimensional regime $d\ge n$ (linear independence of the inputs), but places \emph{no} growth requirement on $m$ as a function of $n$.
This is in contrast with much of the over-parameterized literature, where $m$ is taken to scale at least linearly (and often polynomially) with $n$; see, e.g.,~\cite{soltanolkotabietal18, duetal18, duetal19}.
\end{rmk}

\begin{rmk}[Why the condition $W_1=0$ is not vacuous]\label{rmk:w1zero}
A possible source of confusion is the condition $W_1=0$ (and $b_1=\cdots=b_L=0$) in Theorem~\ref{deepthm}.
At the initialization $w_0$ the network output is indeed constant (in fact, $f(x,w_0)=0$ for all $x$ because $\sigma_\ell(0)=0$), so $S(w_0)=\frac1n\sum_i y_i^2$.
Unless all labels are zero, the loss gradient at $w_0$ is \emph{not} zero. The first-layer derivatives of the network output satisfy (see \eqref{pijfxw}--\eqref{qdef} in \textsection\ref{deepproof})
\[
\frac{\partial}{\partial (W_1)_{rs}} f(x,w_0) = x_s q_r(x,w_0),
\]
where $q_r(x,w_0)$ is a sum of products of downstream weights and activation slopes. At $w_0$ it is positive and independent of $x$.
Consequently, the gradient of $S$ with respect to row $r$ of $W_1$ is
$-(2/n)q_r(w_0)\sum_i y_i x_i$, which is nonzero whenever the labels are not all zero, by linear independence of the inputs.
The role of $W_1=0$ is simply to provide a convenient center point around which we can control the empirical tangent kernel $H(w)$.
\end{rmk}

\begin{rmk}[Why require hidden-layer weights to be positive?]\label{rmk:positive}
The proof goes by finding a lower bound on $\lambda_{\min}(H(w))$ in a neighborhood of the initialization by controlling certain vectors $q_r(x_i,w)$ that appear in the first-layer gradients.
If the hidden-layer weights have mixed signs, then products of weights can cancel and $q_r(x_i,w)$ can be arbitrarily small even when individual weights are not.
Requiring $W_2,\ldots,W_{L-1}$ to have strictly positive entries rules out such cancellations and yields a clean lower bound.
We emphasize that this is a technical assumption used for a worst-case guarantee; it is not meant to model typical random initialization.
\end{rmk}

\begin{rmk}[Activation shifts, biases, and smooth leaky-ReLU]\label{rmk:activations}
The class of activation functions allowed by Theorem~\ref{deepthm} includes many smooth monotone activations used in practice (e.g., $\tanh$ and smooth sigmoids).
The condition $\sigma_\ell(0)=0$ is not restrictive: if an activation $\sigma$ does not satisfy $\sigma(0)=0$, one can replace it by the shifted activation $\tilde\sigma(x):=\sigma(x)-\sigma(0)$ and absorb the resulting constant offset into the biases.
This preserves the represented function class; the convergence statements concern gradient descent in the centered-activation parameterization with the specified zero initial biases. The theorem also applies to smooth approximations of ReLU and leaky-ReLU.
One convenient ``smooth leaky-ReLU'' activation with parameter $a\in(0,1)$ is
\[
\sigma(x) = a x + (1-a)\log(1+e^x) - (1-a)\log 2,
\]
which satisfies $\sigma(0)=0$ and has derivative bounded between $a$ and $1$.
This is the activation used in the experiments in \textsection\ref{subsec:experiments} below.
\end{rmk}

\begin{rmk}[Explicit constants in Theorem~\ref{deepthm}]\label{rmk:constants}
Let $r:=\delta/2$ and $K':=K+r$, with the convention from Theorem~\ref{deepthm} when $L=2$. For $x\ge0$, put
\[
\gamma_\ell(x):=\max\{\sigma_\ell(x),\,|\sigma_\ell(-x)|\}.
\]
Using $R_X$ from the theorem, define the hidden-layer pre-activation bounds
\[
a_1:=r(R_X+1),\qquad
a_\ell:=K'd_{\ell-1}\gamma_{\ell-1}(a_{\ell-1})+r,
\quad 2\le\ell\le L-1,
\]
and set
\[
c_\ell:=\min_{|u|\le a_\ell}\sigma_\ell'(u)>0,\qquad
P:=\prod_{\ell=2}^{L-1}d_\ell,\qquad C:=\prod_{\ell=1}^{L-1}c_\ell.
\]
Empty products are interpreted as $1$. An explicit, unoptimized choice is
\begin{equation}\label{eq:A0-explicit}
A_0:=r+1+2\biggl(\frac{Y_2}{\lambda_X d_1 r^{2L-2}(PC)^2}\biggr)^{1/2}.
\end{equation}
For a chosen initialization, let $A:=\min_j(W_L)_{1j}\ge A_0$. The proof gives
\[
\alpha(w_0,r)\ge\underline\alpha
:=4\lambda_X d_1(A-r)^2r^{2L-4}(PC)^2,
\]
where $\alpha$ is defined for the objective $S$. Since $S(w_0)=Y_2$, this choice yields
$4S(w_0)/(r^2\underline\alpha)<1/4$ if $Y_2>0$.
Thus $\ep=1/4$ is admissible in Theorem~\ref{mainthm1}.
Writing
\begin{align*}
L_1(w_0)&:=\sup_{w\in B(w_0,r)}\max_j|\partial_j S(w)|,\\
L_2(w_0)&:=\sup_{w\in B(w_0,2r)}\max_{j,k}|\partial_j\partial_k S(w)|,
\end{align*}
one may take
\begin{equation}\label{eq:eta-initialization}
\eta_0(w_0):=\min\left\{\frac{r}{L_1(w_0)\sqrt p},\,
\frac{1}{2L_2(w_0)p}\right\}.
\end{equation}
The derivative bounds are evaluated around the actual initialization and therefore depend on its output weights. If $Y_2=0$, the initialization already has zero loss and both dynamics are stationary; any positive step size then suffices.
\end{rmk}

\begin{rmk}[Relation to ``lazy training'' results]\label{rmk:chizat}
A number of general ``lazy training'' results show that, under suitable scaling assumptions, gradient flow remains close to initialization and behaves like kernel regression with the empirical tangent kernel; see, e.g.,~\cite{chizatetal19}.
Our proof is directly in this spirit: we lower bound $\lambda_{\min}(H(w))$ on a neighborhood of an explicit initialization and then apply the local PL criterion from \textsection\ref{introsec}.
Because our scaling and initialization are quite different from the standard large-width settings, it is not immediate to deduce Theorem~\ref{deepthm} as a black-box corollary of existing lazy-training theorems, and we therefore provide a self-contained argument.
\end{rmk}

Finally, note that Theorem~\ref{deepthm} concerns \emph{training error}: it guarantees convergence to a parameter $w^*$ with $S(w^*)=0$ (an interpolating global minimizer of $S$).
It does not by itself imply any \emph{generalization} guarantee.

Our second result about feedforward neural networks, stated below, shows that when initiated from a standard random initialization, the flow converges to a global minimum with positive probability. The proof is in \textsection\ref{newproof}.
\begin{thm}[A conditional corollary for LeCun Gaussian initialization]\label{newdeepthm}
Consider the architecture, data, and activation assumptions of Theorem~\ref{deepthm}, but use the random initialization specified below. Assume additionally that
\[
0<C_1\le\sigma_\ell'(x)\le C_2<\infty
\quad\text{for all }x\in\rr,\quad 1\le\ell\le L-1.
\]
\begin{samepage}
Let $\lambda_X$ and $\Lambda_X$ be the minimum and maximum eigenvalues of $X^TX/n$, and write $Y_2=n^{-1}\sum_i y_i^2$. Define\par
\end{samepage}
\begin{align}\label{eq:random-constants}
P&:=\prod_{\ell=2}^{L-1}d_\ell,\qquad
 a:=4\lambda_Xd_1C_1^{2L-2}P^2,\notag\\
 b&:=3^{L-1}C_2^{L-1}\sqrt{d_1}\,P,\qquad B:=2\Lambda_Xb^2,\notag\\
 \beta&:=\sqrt{\frac{a}{16B}}
 =\sqrt{\frac{\lambda_X}{8\Lambda_X}}
   \left(\frac{C_1}{3C_2}\right)^{L-1}.
\end{align}
As before, an empty product is $1$. Fix $c>0$ and choose
\begin{equation}\label{eq:random-M}
M\ge\max\left\{1,\left(\frac{32Y_2}{a}\right)^{1/(2L)}\right\}.
\end{equation}
Initialize all biases to zero and draw the entries of each $W_\ell$ independently from $\mathcal N(0,c/d_{\ell-1})$. Set
\begin{align}\label{eq:random-event}
E_M&:=\{(W_\ell)_{ij}\in[2M,3M]\text{ for all }\ell=2,\ldots,L\text{ and all }i,j\},\notag\\
F_M&:=\{\|W_1\|_F\le\beta M\},
\end{align}
where $\|\cdot\|_F$ denotes the Frobenius norm. On $E_M\cap F_M$, gradient flow converges to a point $w^*$ with $S(w^*)=0$. There is also a deterministic $\eta_0>0$, depending on the fixed data, architecture, activation functions, $C_1,C_2$, and $M$, such that gradient descent with any $0<\eta\le\eta_0$ converges to such a point on this same event. This step-size bound can be chosen uniformly over the event, for these fixed data and functions. The event $E_M\cap F_M$ has strictly positive probability.
\end{thm}

\begin{rmk}[Interpretation of Theorem~\ref{newdeepthm}]\label{rmk:random-probability}
The first-layer threshold is $\beta M$, rather than $M$. Appendix~\ref{newproof} verifies that this choice and~\eqref{eq:random-M} imply the trapping inequality with a uniform margin.
Independence of the weight matrices gives
\begin{align*}
\pp(E_M\cap F_M)
&=\pp\left(\chi^2_{dd_1}\le\frac{d\beta^2M^2}{c}\right)\\
&\quad\times\prod_{\ell=2}^L
\left[\Phi\left(3M\sqrt{d_{\ell-1}/c}\right)
      -\Phi\left(2M\sqrt{d_{\ell-1}/c}\right)\right]^{d_\ell d_{\ell-1}}>0,
\end{align*}
where $\Phi$ is the standard normal distribution function and $\chi^2_k$ denotes a chi-squared random variable with $k$ degrees of freedom. This is an explicit lower bound on the probability of convergence. It can be very small, even for small fixed networks if the required $M$ is large. Thus the theorem gives a positive-probability guarantee, not a high-probability guarantee for typical Gaussian initialization. As $M$ tends to infinity, $\pp(F_M)$ tends to $1$ while the downstream positive-box probability tends to $0$.
\end{rmk}


\section{Numerical experiments}\label{subsec:experiments}
One appealing aspect of the neural network argument is that it produces a concrete initialization strategy rather than only an abstract convergence guarantee.
In this section we illustrate this ``theory-guided initialization'' numerically on small synthetic examples.
The experiments are not meant as performance benchmarks; rather, they show that the proof-motivated starting point can have a visible practical impact on optimization behavior at fixed (and relatively small) width.


\paragraph{Setup.}
We trained a depth-$3$ network (two hidden layers) with the smooth leaky-ReLU activation from Remark~\ref{rmk:activations} with $a=0.1$.
We used full-batch gradient descent on~\eqref{swdef} with step size $\eta=10^{-3}$ for $4000$ iterations, and we report the mean training loss over five random seeds.

\paragraph{Initializations compared.}
We compared:
(i) the deterministic ``positive'' initialization suggested by Theorem~\ref{deepthm} (set $W_1=0$, set each entry of $W_2$ to $\delta=0.5$, and set each entry of $W_3$ to $A=1$);
(ii) a standard LeCun Gaussian initialization (all weights i.i.d.\ $\mathcal{N}(0,1/d_{\ell-1})$);
(iii) a simple positive variant where the Gaussian weights are replaced by their absolute values (entrywise);
and (iv) a ``conditional-event'' variant where $W_2,W_3$ are sampled from the LeCun Gaussian \emph{conditioned} to lie entrywise in a positive box $[2M,3M]$ (equivalently, sampled from a truncated Gaussian supported on $[2M,3M]$), mimicking the event discussed around Theorem~\ref{newdeepthm}.



\paragraph{Baseline result (width $d_1=d_2=5$).}
Figure~\ref{fig:trainingloss} plots the mean training loss (over five random seeds) versus iteration on a log scale.
The theory-guided ``positive'' initialization exhibits a clear linear rate over several orders of magnitude.
By contrast, at this fixed width the standard Gaussian initializations do not consistently reach zero training error, consistent with the discussion around Theorem~\ref{newdeepthm}: the random-initialization guarantee there is conditional on an event that may have very small probability at finite width.
(The conditional-event curve (iv) shows that, \emph{given} the required positivity event, behavior is much closer to the deterministic positive initialization.)

\begin{figure}[t]
\centering
\includegraphics[width=0.78\textwidth]{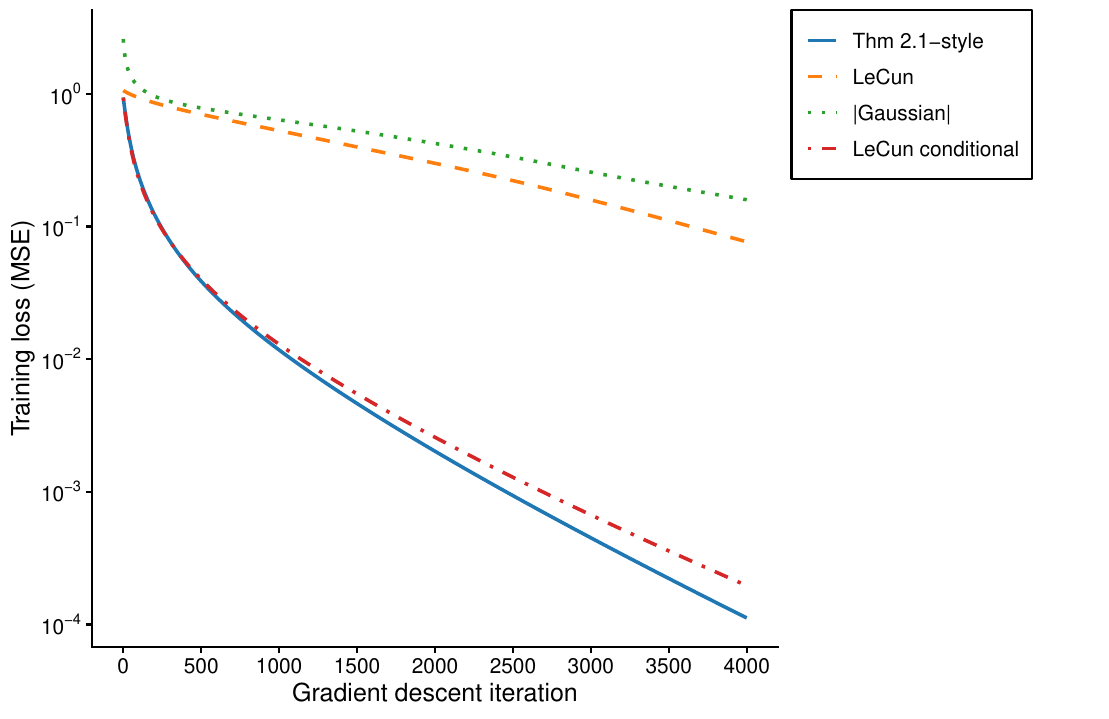}
\caption{Mean training loss (MSE) versus gradient descent iteration in the baseline experiment. The y-axis is logarithmic. The legend is placed outside the plotting region and the curves use distinct line types so that color is not essential.}
\label{fig:trainingloss}
\end{figure}

\paragraph{Wider network.}
To see how behavior changes with width, we repeated the same experiment with widths $d_1=d_2=10$.
Figure~\ref{fig:trainingloss_w10} shows that the qualitative picture persists: the proof-motivated positive initialization remains reliably fast, while random initializations improve with width but can still be noticeably slower at these sizes.

\begin{figure}[t]
\centering
\IfFileExists{fig_training_loss_width10.pdf}{
\includegraphics[width=0.78\textwidth]{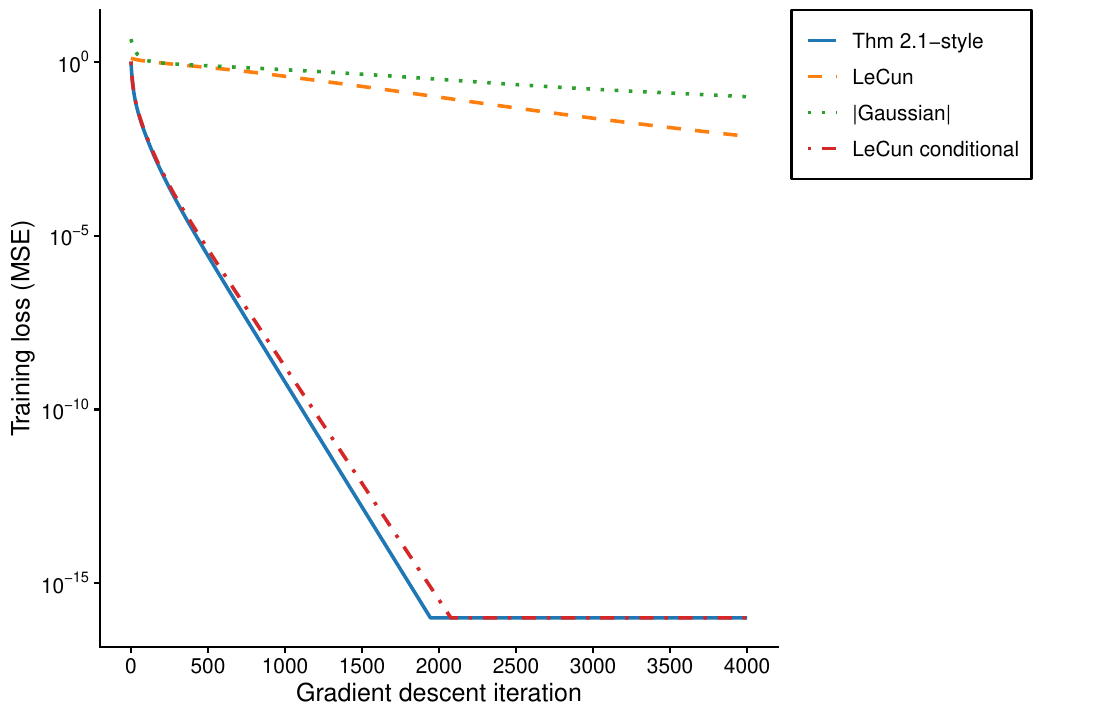}
}{
{\small \texttt{fig\_training\_loss\_width10.pdf} not found. Run the accompanying R script to generate it.}
}
\caption{Mean training loss (MSE) versus gradient descent iteration for the same synthetic setup, but with widths $d_1=d_2=10$. The y-axis is logarithmic.}
\label{fig:trainingloss_w10}
\end{figure}

\paragraph{Output-layer scaling.}
Theorem~\ref{deepthm} requires the output-layer weights to be ``sufficiently large.'' To visualize this dependence, we fixed the positive initialization pattern and swept the output-layer scale $A$ over a range of values.
Figure~\ref{fig:trainingloss_Asweep} shows a clear effect of the scale on optimization speed: once $A$ is large enough, the loss decays rapidly and approximately linearly on a log scale, whereas smaller $A$ can substantially slow down progress.
This provides a simple example of how a convergence proof can translate into a practically useful initialization guideline.

\begin{figure}[t]
\centering
\IfFileExists{fig_training_loss_A_sweep.pdf}{
\includegraphics[width=0.78\textwidth]{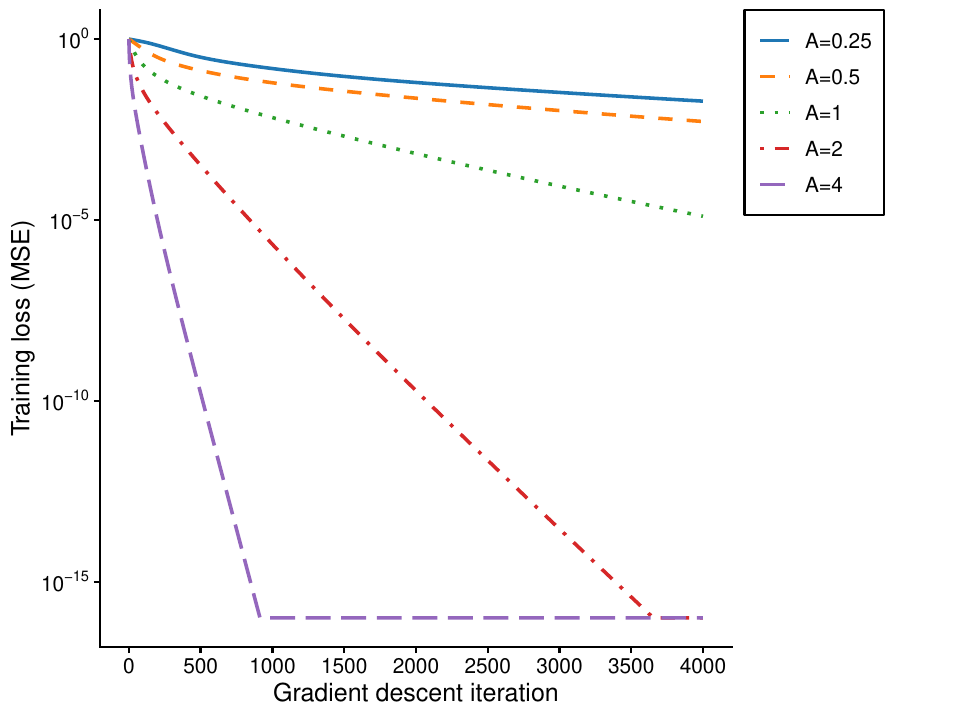}
}{
{\small \texttt{fig\_training\_loss\_A\_sweep.pdf} not found. Run the accompanying R script to generate it.}
}
\caption{Mean training loss (MSE) versus iteration for the positive initialization with different output-layer scales $A$. The y-axis is logarithmic.}
\label{fig:trainingloss_Asweep}
\end{figure}

\FloatBarrier
\section*{Acknowledgements}
I thank Persi Diaconis, Dmitriy Drusvyatskiy, John Duchi, and Lexing Ying for useful feedback, references, and suggestions. The research was partially supported by NSF grant DMS-2413864.

\bibliographystyle{abbrvnat}

\bibliography{myrefs} 

\appendix

\setcounter{section}{0}

\refstepcounter{section}   
\section*{Appendix}        
\addcontentsline{toc}{section}{Appendix}

\subsection{Proof of Theorem \ref{mainthm2}}\label{thm2proof}
To avoid trivialities, let us assume that $f$ is not everywhere zero in $B(x_0,r)$, and hence $\alpha<\infty$. 
Throughout this proof, we will denote the closed ball $B(x_0,r)$ by $B$ and its interior by $U$.  We start with some supporting lemmas. Many of these are standard results, but we prove them anyway for the sake of completeness and to save the trouble of sending the reader to look at references. 
For each $x\in \rr^p$, let $T_x$ denote the supremum of all finite $T$ such that for all $S\le T$, there is a unique continuous function $\phi_x:[0,S]\to \rr^p$ satisfying the integral equation
\[
\phi_x(t) = x - \int_0^t\nabla f(\phi_x(s))ds.
\]
Since there is always such a unique function for $T=0$ (namely, $\phi_x(0)=x$), $T_x$ is well-defined. From the above  definition, we see that there is a unique continuous function $\phi_x:[0,T_x)\to \rr^p$ satisfying the above integral equation. We will refer to this $\phi_x$ as the gradient flow starting from $x$.
\begin{lmm}\label{triviallmm}
Take any $x\in \rr^p$ and any $t\in [0,T_x)$ (assuming that $T_x>0$). Let $y = \phi_x(t)$. Then $T_y = T_x - t$ (with the convention $\infty-t = \infty$), and $\phi_y(s) = \phi_x(t+s)$ for all $s\in [0,T_y)$.
\end{lmm}
\begin{proof}
Take any $S\in [0, T_x-t)$. Define $g(s) := \phi_x(t+s)$ for $s\in [0,S]$. Then 
\begin{align*}
g(s) &= \phi_x(t+s) = x - \int_0^{t+s} \nabla f(\phi_x(u)) du\\
&= \phi_x(t)  - \int_t^{t+s} \nabla f(\phi_x(u)) du= y - \int_0^s \nabla f(g(v)) dv.
\end{align*}
Thus, $g$ satisfies the gradient flow integral equation starting from $y$, in the interval $[0,S]$. Let $h$ be any other flow with these properties. Define $w: [0,t+S] \to \rr^p$ as
\[
w(u) := 
\begin{cases}
\phi_x(u) &\text{ if } u\le t,\\
h(u-t) &\text{ if } u > t. 
\end{cases}
\]
Then for $u\in [t, t+S]$,
\begin{align*}
w(u) = h(u-t) &= y - \int_0^{u-t} \nabla f(h(v)) dv\\
&= \phi_x(t) - \int_0^{u-t} \nabla f(w(t+v)) dv\\
&= x - \int_0^t \nabla f(w(s)) ds - \int_t^{u} \nabla f(w(s)) ds\\
&= x - \int_0^u \nabla f(w(s)) ds.
\end{align*}
Obviously, $w$ satisfies this integral equation also in $[0,t]$. Thus, the uniqueness of the flow in $[0,t+S]$ implies that $w = \phi_x$ in $[0, t+S]$. This implies, in particular, that $h=g$. Therefore, we conclude that $T_y \ge T_x-t$. This also shows that $\phi_y(s) = \phi_x(t+s)$ for all $s\in [0, T_x-t)$. So it only remains to show that $T_y \le T_x -t$. Suppose not. Then $T_x$ must be finite, and there is some $S> T_x-t$ such that the gradient flow integral equation starting from $y$ has a unique solution in $[0, L]$ for every $L\le S$. Take any $L\le S$. Define $g: [0, t+L]\to \rr^p$ as 
\[
g(s) :=
\begin{cases}
\phi_x(s) &\text{ if } s \le t,\\
\phi_y(s-t) &\text{ if } s> t.
\end{cases}
\]
Then for any $s\in [t, t+L]$,
\begin{align*}
g(s) &= y - \int_0^{s-t} \nabla f(\phi_y(u)) du\\
&= x - \int_0^t \nabla f (g(u)) du -  \int_0^{s-t} \nabla f(g(u+t)) du\\
&= x - \int_0^s \nabla f(g(u)) du.
\end{align*}
Thus, $g$ satisfies the integral equation for the gradient flow starting from $x$ in the interval $[t,t+L]$. Obviously, it satisfies the equation also in $[0,t]$. Suppose that $h:[0, t+L]\to \rr^p$ is any other continuous function satisfying this property. Then by the uniqueness of the gradient flow starting from $x$ in the interval $[0,t]$, we get that $h(s) = \phi_x(s)=g(s)$ for all $s\le t$. Next, define $h_1 :[0, L] \to\rr^p$ as $h_1(u) := h(t+u)$. Then, since $h=\phi_x$ on $[0,t]$, 
\begin{align*}
h_1(u) &= h(t+u) \\
&= x - \int_0^{t+u}\nabla f(h(v)) dv\\
&= y - \int_t^{t+u} \nabla f(h(v)) dv = y - \int_0^{u} \nabla f(h_1(s)) ds.
\end{align*}
Thus, $h_1$ satisfies the integral equation for the gradient flow starting from $y$ in the interval $[0, L]$. Thus, by the uniqueness assumption for this flow in $[0,L]$, we get that $h_1 = \phi_y$ in this interval. Consequently, $h = g$ in $[0, t+L]$. Thus, the gradient flow starting from $x$ has a unique solution in $[0,T]$ for every $T\le t+S$. Since $t+S> T_x$, this contradicts the definition of $T_x$. 
\end{proof}
\begin{lmm}\label{standardlmm}
For any compact set $K\subseteq \rr^p$, $\inf_{x\in K} T_x>0$.
\end{lmm}
\begin{proof}
Let $K'$ be the set of all points that are within distance $1$ from $K$. Note that $K'\supseteq K$ and $K'$ is compact.  Since $f$ is in $C^2$, its first and second order derivatives are uniformly bounded on the compact convex hull of $K'$. Thus, there is some $L$ such that for any $x,y\in K'$, 
\[
|\nabla f(x)|\le \frac{L}{2}, \ \text{ and } \  |\nabla f(x)-\nabla f(y)|\le \frac{L}{2}|x-y|,
\]
where $|\cdot|$ denotes Euclidean norm. Take any $x\in K$ and any $T\le 1/L$. Let $\mb$ denote the Banach space of all continuous maps from $[0,T]$ into $\rr^p$, equipped with the norm
\[
\|g\| := \sup_{t\in [0,T]} |g(t)|.
\]
Let $\ma$ be the subset of $\mb$ consisting of all $g$ such that $g(0)=x$ and $g(t)\in K'$ for all $t\in [0,T]$. It is easy to see that $\ma$ is a closed subset of $\mb$. Define a map $\Phi:\ma \to \mb$ as
\[
\Phi(g)(t) := x - \int_0^t \nabla f(g(s)) ds.
\]
Then note that for any $t\in [0,T]$,
\[
|\Phi(g)(t)-x| \le \int_0^t|\nabla f(g(s))| ds\le \int_0^t \frac{L}{2} ds \le \frac{1}{2},
\]
where the second inequality holds because $g(s)\in K'$ for all $s$, and the third inequality holds because $t\le 1/L$. Since all points at distance $\le 1$ from $x$ are in $K'$, this shows that $\Phi(\ma)\subseteq \ma$. Next, note that for any $g,h\in \ma$, and any $t\in [0,T]$,
\begin{align*}
|\Phi(g)(t)-\Phi(h)(t)| &\le \int_0^t|\nabla f(g(s)) - \nabla f(h(s))| ds\\
&\le \frac{L}{2}\int_0^t |g(s)-h(s)|ds \\
&\le \frac{1}{2}\|g-h\|,
\end{align*}
where the second inequality holds because $g(s),h(s)\in K'$, and the third inequality holds because $t\le 1/L$. This proves that $\Phi$ is a contraction mapping on $\ma$, and therefore, by the Banach fixed point theorem, it has a unique fixed point $g^*\in \ma$. Then $g^*$ is a continuous map from $[0,T]$ into $\rr^p$, that satisfies the integral equation
\[
g^*(t) = x - \int_0^t \nabla f(g^*(s)) ds.
\]
We claim that $g^*$ is the only such map. To prove this, suppose that there exists another map $h$ with the above properties. If $h$ maps into $K'$, then the uniqueness of the fixed point implies that $h=g^*$. So, suppose that $h$ ventures outside $K'$. Let 
\[
t_0 := \inf\{t: h(t) \notin K'\}.
\]
Since $h(0)=x\in K'$ and $h$ goes outside $K'$, $t_0$ is well-defined and finite. Moreover, since $K'$ is closed, $h(t_0)\in \partial K'$. But note that since $h(s)\in K'$ for all $s\le t_0$,  
\begin{align*}
|h(t_0) - x|&\le \int_0^{t_0} |\nabla f(h(s))| ds \le \frac{L t_0}{2}\le \frac{1}{2}.
\end{align*}
But this implies that the ball of radius $1/3$ centered at $h(t_0)$ is completely contained in $K'$, and hence, $h(t_0)\notin \partial K'$. Thus, $h$ cannot venture outside $K'$. We conclude that $T_x \ge 1/L$.
\end{proof}
The above lemma has several useful corollaries.
\begin{cor}\label{badcor}
If $T_x<\infty$ for some $x\in \rr^p$, then $|\phi_x(t)|\to\infty$ as $t\to T_x$.
\end{cor}
\begin{proof}
Suppose that $|\phi_x(t)|\not\to \infty$. Then there is a compact set $K$ such that for any $\ep>0$, we can find $t\in [T_x-\ep,T_x)$ with $\phi_x(t)\in K$. By Lemma \ref{standardlmm}, we can find $\ep\in (0, \inf_{y\in K} T_y)$. Take any $t\in [T_x-\ep,T_x)$ such that $y:= \phi_x(t)\in K$. Then, by Lemma~\ref{triviallmm}, we have $T_y = T_x - t < \ep$. But  this contradicts the fact that $T_y\ge \inf_{z\in K}T_z >\ep$.
\end{proof}
\begin{cor}\label{trivialcor}
If $\nabla f(x)=0$ or $f(x)=0$ for some $x\in \rr^p$, then $T_x=\infty$ and $\phi_x(t)=x$ for all $t\in [0,\infty)$.
\end{cor}
\begin{proof}
If $f(x)=0$, then since $f$ is a nonnegative $C^2$ function, $\nabla f(x)$ must also be zero. Thus, it suffices to prove the result under the assumption that $\nabla f(x)=0$. 
Taking $K=\{x\}$ in Lemma \ref{standardlmm}, we get that $T_x>0$. Suppose that $T_x$ is finite. Then let $t:=T_x/2$. Note that since $\nabla f(x)=0$, $g(s)\equiv x$ is a solution of the gradient flow equation starting from $x$ in the interval $[0,t]$. By uniqueness, this shows that $\phi_x(s)=x$ for all $s\in [0,t]$. In particular, $\phi_x(t)=x$. By Lemma \ref{triviallmm}, this implies that $T_x = T_x-t = T_x/2$, which contradicts the fact that $T_x$ is nonzero and finite. Thus, $T_x=\infty$. Again, the function $g(t)\equiv x$ is a solution of the flow equation in $[0,\infty)$. Thus, by uniqueness, $\phi_x(t)=x$ for all $t\ge0$.
\end{proof}
\begin{cor}\label{impcor}
Take any $x\in \rr^p$. If for some $t\in [0,T_x)$, we have $\nabla f(\phi_x(t)) = 0$ or $f(\phi_x(t)) = 0$, then $T_x=\infty$ and $\phi_x(s)=\phi_x(t)$ for all $s>  t$. 
\end{cor}
\begin{proof}
Take any $t$ such that $\nabla f(\phi_x(t)) = 0$ or $f(\phi_x(t))=0$. Let $y := \phi_x(t)$. Then by Lemma \ref{triviallmm}, $T_y = T_x - t$. But by Corollary \ref{trivialcor}, $T_y = \infty$. Thus, $T_x=\infty$. Also by Corollary \ref{trivialcor}, $\phi_y(u)=y = \phi_x(t)$ for all $u> 0$, and by Lemma \ref{triviallmm}, $\phi_y(u) = \phi_x(t+u)$ for all $u>0$. Thus, $\phi_x(s)=\phi_x(t)$ for all $s>t$.
\end{proof}
In the following, let us fix $x_0$ and $r$ as in Theorem \ref{mainthm2} and write $\phi$ and $T$ instead of $\phi_{x_0}$ and $T_{x_0}$, for simplicity of notation.
\begin{lmm}\label{phicannot}
If $T<\infty$, then $\phi$ must visit the boundary of $B$. 
\end{lmm}
\begin{proof}
Suppose that $T<\infty$ and $\phi$ remains in $U$ throughout. By Lemma \ref{standardlmm}, there is some $\delta>0$ such that $T_y\ge \delta$ for all $y\in B$. Choose $t\in (T-\delta, T)$. Let $z:= \phi(t)$. Then  $z\in U$, and therefore $T_z\ge \delta$. This shows that the gradient flow starting from $x_0$ exists up to time at least $t+\delta$. Since $t+\delta > T$, this gives a contradiction which proves that $\phi$ cannot remain in $U$ throughout.
\end{proof}
\begin{lmm}\label{expboundlmm}
Let $t\ge 0$ be any number such that $\phi(s)\in B$ for all $s\le t$. Then for all $s\le t$,
\[
f(\phi(s)) \le e^{-\alpha s} f(x_0).
\]
\end{lmm}
\begin{proof}
Note that by the flow equation for $\phi$,
\begin{align}\label{ddsfphi}
\frac{d}{ds}f(\phi(s)) = \nabla f(\phi(s)) \cdot \phi'(s) = - |\nabla f(\phi(s))|^2. 
\end{align}
By the definition of $\alpha$, the right side is bounded above by $-\alpha f(\phi(s))$ if $\phi(s)\in B$. It is now a standard exercise to deduce the claimed inequality.
\end{proof}
\begin{lmm}\label{phicannot2}
The flow $\phi$ cannot visit the boundary of $B$.
\end{lmm}
\begin{proof}
Suppose that $\phi$ does visit the boundary of $B$. Let $t_0:= \inf\{t:\phi(t)\in \partial B\}$. Then $\phi(t_0)\in \partial B$ and $\phi(s)\in U$ for all $s< t_0$. By the flow equation,
\begin{align}\label{phit0x0}
|\phi(t_0)-x_0| &\le \int_0^{t_0} |\nabla f(\phi(s))| ds.
\end{align}
Now, since $\phi(s)\ne \phi(t_0)$ for all $s<t_0$, Corollary \ref{impcor} shows that $f(\phi(s))\ne 0$ for all $s< t_0$. Thus, the map 
\[
g(s) := \sqrt{f(\phi(s))}
\]
is differentiable in $(0, t_0)$, with continuous derivative
\begin{align*}
g'(s)&= \frac{1}{2g(s)} \frac{d}{ds}f(\phi(s))= -\frac{|\nabla f(\phi(s))|^2}{2g(s)},
\end{align*}
where the second identity follows from \eqref{ddsfphi}. Thus, for any compact interval $[a,b]\subseteq (0,t_0)$, 
\[
\int_a^{b} \frac{|\nabla f(\phi(s))|^2}{2g(s)} ds = -\int_a^b g'(s)ds = g(a)-g(b).
\]
Since $g$ is continuous on $[0,t_0]$, we can take $a\to 0$ and $b\to t_0$, and apply the monotone convergence theorem on the left, to get
\begin{align*}
\int_0^{t_0} \frac{|\nabla f(\phi(s))|^2}{2g(s)} ds &= g(0)-g(t_0)\le g(0)=\sqrt{f(x_0)}.  
\end{align*}
Therefore, by \eqref{phit0x0} and the Cauchy--Schwarz inequality, 
\begin{align*}
|\phi(t_0)-x_0| &\le \biggl(\int_0^{t_0} 2g(s) ds \int_0^{t_0} \frac{|\nabla f(\phi(s))|^2}{2g(s)}  ds \biggr)^{1/2}\\
&\le f(x_0)^{1/4} \biggl(\int_0^{t_0} 2g(s) ds\biggr)^{1/2}.
\end{align*}
On the other hand, since $\phi(s)\in B$ for all $s\le t_0$, Lemma \ref{expboundlmm} shows that 
\[
g(s)\le e^{-\alpha s/2} \sqrt{f(x_0)}
\]
for all $s\le t_0$. Plugging this bound into the previous display, we get
\begin{align*}
|\phi(t_0)-x_0| &\le \sqrt{f(x_0)} \biggl(\int_0^{t_0} 2e^{-\alpha s/2}ds \biggr)^{1/2}\\
&\le\sqrt{f(x_0)} \biggl(\int_0^\infty 2e^{-\alpha s/2}ds \biggr)^{1/2}= 2\sqrt{\frac{f(x_0)}{\alpha}}. 
\end{align*}
But the last quantity is strictly less than $r$, by assumption \eqref{mainassump}. Since $\phi(t_0)\in \partial B$, this gives a contradiction, which proves the lemma.
\end{proof}
We are now ready to prove Theorem \ref{mainthm2}.
\begin{proof}[Proof of Theorem \ref{mainthm2}]
Combining Lemma \ref{phicannot} and Lemma \ref{phicannot2}, we see that $T=\infty$. Moreover by Lemma \ref{phicannot2}, $\phi$ stays in $U$ forever. Therefore, by Lemma \ref{expboundlmm}, it follows that $f(\phi(s))\le e^{-\alpha s} f(x_0)$ for all $s$. It remains to establish the convergence of the flow and the rate of convergence. Let 
\[
t_0 := \inf\{t: f(\phi(t)) = 0\},
\]
with the understanding that $t_0=\infty$ if $f(\phi(t)) >0$ for all $t$. Then note that for any $0\le s < t< t_0$,
\begin{align*}
|\phi(t)-\phi(s)|&\le \int_s^{t} |\nabla f(\phi(u))| du\\
&\le \biggl(\int_s^{t} 2g(u) du \int_s^{t} \frac{|\nabla f(\phi(u))|^2}{2g(u)}  du \biggr)^{1/2},
\end{align*}
where $g(u) = \sqrt{f(\phi(u))}$, as in the proof of Lemma \ref{phicannot2}. As in that proof, note that
\begin{align*}
\int_s^{t} \frac{|\nabla f(\phi(u))|^2}{2g(u)}  du &= g(s)-g(t) \le g(s) = \sqrt{f(\phi(s))}\le e^{-\alpha s/2} \sqrt{f(x_0)}.
\end{align*}
On the other hand,
\[
\int_s^{t} 2g(u) du \le \sqrt{f(x_0)} \int_s^t 2e^{-\alpha u/2}du \le \frac{4\sqrt{f(x_0)}}{\alpha} e^{-\alpha s/2}.
\]
Combining the last three displays, and invoking assumption \eqref{mainassump}, we get
\begin{align}\label{phicauchy}
|\phi(t)-\phi(s)| &\le 2\sqrt{\frac{f(x_0)}{\alpha}} e^{-\alpha s/2}\le re^{-\alpha s/2}.
\end{align}
Note that the bound does not depend on $t$. If $t_0<\infty$, then by Corollary \ref{impcor}, $\phi(t)=\phi(t_0)$ for all $t> t_0$. Thus, in this case $\phi(t)$ converges to $x^* := \phi(t_0)$. The rate of convergence is established by taking $t\to t_0$ in \eqref{phicauchy}. If $t_0=\infty$, then~\eqref{phicauchy} proves the Cauchy property of the flow $\phi$, which shows that $\phi(t)$ converges to some $x^*\in B$ as $t\to\infty$. Again, taking $t\to\infty$ in \eqref{phicauchy} proves the rate of convergence.
\end{proof}

\subsection{Proof of Theorem \ref{mainthm1}}\label{thm1proof}
If $f(x_0)=0$, then $\nabla f(x_0)=0$ and therefore $x_k=x_0$ for all $k$, and there is nothing to prove. So, let us assume that $f(x_0)>0$. 
The following lemma is the key step in the proof of Theorem \ref{mainthm1}. 
\begin{lmm}\label{mainlmm1}
For all $k\ge 0$, $x_k\in B(x_0,r)$.
\end{lmm}
The proof of this lemma will be carried out via induction on $k$. We have $x_0\in B(x_0,r)$. Suppose that $x_1,\ldots,x_{k-1}\in B(x_0,r)$ for some $k\ge 1$. We will use this hypothesis to show that $x_k\in B(x_0,r)$, with $k$ remaining fixed henceforth.
\begin{lmm}\label{rjboundlmm}
For each $j$, define 
\begin{align}\label{rjdef}
R_j := f(x_{j+1}) - f(x_{j}) + \eta|\nabla f(x_{j})|^2.
\end{align}
Then for all $0\le j\le k-1$, $|R_j|\le \ep \eta |\nabla f(x_j)|^2$. 
\end{lmm}
\begin{proof}
Since $x_{k-1}\in B(x_0,r)$, the assumed upper bound on $\eta$ implies that 
\[
\eta|\nabla f(x_{k-1})| \le L_1\sqrt{p} \eta \le r. 
\]
Thus, $|x_{k}-x_{k-1}|\le r$, and therefore, $x_{k}\in B(x_0,2r)$. Consequently, the line segment joining $x_{k-1}$ and $x_k$ lies entirely in $B(x_0,2r)$. Since $x_0,\ldots,x_{k-1}\in B(x_0,r)$, the line segments joining $x_{j}$ and $x_{j+1}$ lie in $B(x_0,r)$ for all $j\le k-2$. Thus, by Taylor expansion, we have that for any $0\le j\le k-1$, 
\begin{align*}
R_j &= f(x_{j} - \eta\nabla f(x_j)) - f(x_{j}) + \eta|\nabla f(x_{j})|^2\\
&= \frac{\eta^2}{2} \nabla f(x_j) \cdot \nabla^2f(x_j^*) \nabla f(x_j),
\end{align*}
where $x_j^*$ is a point on the line segment joining $x_j$ and $x_{j+1}$, and $\nabla^2f(x_j^*)$ is the Hessian matrix of $f$ at $x_j^*$. Since $L_2$ is an upper bound on the magnitudes of all second order derivatives of $f$ in $B(x_0,2r)$, and $x_j^*\in B(x_0,2r)$, this gives
\begin{align*}
|R_j| &\le\frac{ L_2\eta^2}{2} \sum_{i,i'=1}^p |\partial_i f(x_j)||\partial_{i'} f(x_j)|\\
&=\frac{ L_2\eta^2}{2} \biggl(\sum_{i=1}^p |\partial_i f(x_j)|\biggr)^2 \le \frac{L_2\eta^2 p}{2} \sum_{i=1}^p |\partial_i f(x_j)|^2,
\end{align*}
which proves that $|R_j|\le \frac{1}{2}L_2p\eta^2 |\nabla f(x_j)|^2$. Since $L_2p\eta \le 2\ep$, this proves the claim.
\end{proof}
\begin{lmm}\label{fxjbound}
We have $(1-\ep)\alpha \eta \le 1$, and for any $0\le j\le k$, 
\begin{align*}
f(x_j) &\le (1-(1-\ep)\alpha \eta)^jf(x_0). 
\end{align*}
\end{lmm}
\begin{proof}
Since $x_j\in B(x_0,r)$ for $j\le k-1$, Lemma \ref{rjboundlmm} and the definition of $\alpha$ imply that for $j\le k-1$, 
\begin{align*}
f(x_{j+1}) &= f(x_j) - \eta |\nabla f(x_j)|^2 + R_j\\
&\le f(x_j) -  (1-\ep) \eta |\nabla f(x_j)|^2\\
&\le ( 1- (1-\ep)\alpha \eta)f(x_j).
\end{align*}
Since $f(x_0)>0$ and $f(x_1)\ge 0$, we can take $j=0$ above and divide both sides by $f(x_0)$ to get that $(1-\ep) \alpha \eta \le 1$. Iterating the above inequality gives the desired upper bound for $f(x_j)$.
\end{proof}
\begin{lmm}\label{rjnew}
For each $j\le k-1$, 
\[
f(x_j) - f(x_{j+1}) \ge (1-\ep) \eta |\nabla f(x_j)|^2. 
\]
\end{lmm}
\begin{proof}
By Lemma \ref{rjboundlmm},
\begin{align*}
\eta |\nabla f(x_j)|^2 &= f(x_j) - f(x_{j+1}) + R_j\\
&\le f(x_j)-f(x_{j+1}) + \ep \eta |\nabla f(x_j)|^2.
\end{align*}
Rearranging terms, we get the desired inequality.
\end{proof}
\begin{lmm}\label{mainsub1}
For all $0\le j\le k-1$, 
\begin{align*}
\sum_{l=j}^{k-1}\eta |\nabla f(x_l)| &\le  (1-(1-\ep)\alpha \eta)^{j/2}\sqrt{\frac{4f(x_0)}{\alpha (1-\ep)^2}}.  
\end{align*}
\end{lmm}
\begin{proof}
Let $\kappa := (1-\ep)^{-1}$. By Lemma \ref{rjnew},
\begin{align*}
\eta |\nabla f(x_j)| = \sqrt{\eta^2 |\nabla f(x_j)|^2} &\le \sqrt{\kappa \eta (f(x_j)-f(x_{j+1}))}.
\end{align*}
Also by Lemma \ref{rjnew}, $f(x_j)\ge f(x_{j+1})$ for each $j\le k-1$. Thus, for any $j\le k-1$, we use the Cauchy--Schwarz inequality to get
\begin{align*}
&\sum_{l=j}^{k-1}\eta |\nabla f(x_l)| \le  \sum_{l=j}^{k-1}  \sqrt{\kappa\eta(\sqrt{f(x_l)}+ \sqrt{f(x_{l+1})}) (\sqrt{f(x_l)} - \sqrt{f(x_{l+1})})}\\
&\le \biggl(\sum_{l=j}^{k-1} \kappa \eta (\sqrt{f(x_l)}+\sqrt{f(x_{l+1})}) \sum_{l=j}^{k-1} (\sqrt{f(x_l)}- \sqrt{f(x_{l+1})})\biggr)^{1/2}\\
&= \biggl((\sqrt{f(x_j)} - \sqrt{f(x_k)})\sum_{l=j}^{k-1} \kappa \eta(\sqrt{f(x_l)}+\sqrt{f(x_{l+1})})\biggr)^{1/2}\\
&\le \biggl(2\kappa\eta \sqrt{f(x_j)} \sum_{l=j}^{k-1}  \sqrt{f(x_l)}\biggr)^{1/2}. 
\end{align*}
Using Lemma \ref{fxjbound} to bound the terms  on the right side, we get
\begin{align*}
\sum_{l=j}^{k-1}\eta |\nabla f(x_l)| &\le (1-(1-\ep)\alpha \eta)^{j/4} \sqrt{f(x_0)} \biggl(2\kappa\eta \sum_{l=j}^{k-1}(1-(1-\ep)\alpha \eta)^{l/2}\biggr)^{1/2}. 
\end{align*}
Now, for $t\in [0,1]$, we have the inequality $1-t\le (1-t/2)^2$. This gives 
\begin{align*}
 \sum_{l=j}^{k-1}(1-(1-\ep)\alpha \eta)^{l/2} &\le  (1-(1-\ep)\alpha \eta)^{j/2}\sum_{q=0}^\infty(1-(1-\ep)\alpha \eta)^{q/2} \\
 &\le (1-(1-\ep)\alpha \eta)^{j/2} \sum_{q=0}^{\infty}\biggl(1-\frac{(1-\ep)\alpha \eta}{2}\biggr)^{q}\\
 &= \frac{2(1-(1-\ep)\alpha \eta)^{j/2}}{(1-\ep)\alpha \eta}.
\end{align*}
Plugging this into the previous display completes the proof.
\end{proof}
We are now ready to complete the proof of Lemma \ref{mainlmm1} and then use it to prove Theorem~\ref{mainthm1}.
\begin{proof}[Proof of Lemma \ref{mainlmm1}]
Applying Lemma \ref{mainsub1} with $j=0$, and recalling that $\rho<1-\ep$ by construction, we get
\begin{align*}
|x_k - x_0| &\le \sum_{j=0}^{k-1}|x_{j+1}-x_{j}| = \sum_{j=0}^{k-1}\eta |\nabla f(x_j)| \le \sqrt{\frac{4f(x_0)}{\alpha(1-\ep)^2}} < r. 
\end{align*}
This proves that $x_k\in B(x_0,r)$, completing the induction step. 
\end{proof}
\begin{proof}[Proof of Theorem \ref{mainthm1}]
By Lemma \ref{mainlmm1}, we know that $x_k\in B(x_0,r)$ for all $k$. Thus, Lemma~\ref{mainsub1} holds for all $k\ge 1$ and all $j<k$. In particular, 
\begin{align*}
|x_k - x_j| &\le \sum_{l=j}^{k-1} |x_{l+1}-x_l|= \sum_{l=j}^{k-1}\eta |\nabla f(x_l)| \\
&\le  (1-(1-\ep)\alpha \eta)^{j/2}\sqrt{\frac{4f(x_0)}{\alpha(1-\ep)^2}} \le r(1-(1-\ep)\alpha \eta)^{j/2}.
\end{align*}
Note that the bound goes to zero as $j\to\infty$, and has no dependence on $k$. Thus, $\{x_k\}_{k\ge 0}$ is a Cauchy sequence in $B(x_0,r)$, and therefore, converges to a limit $x^*\in B(x_0,r)$. Moreover, the above bound is also a bound for $|x^*-x_j|$, since it has no dependence on $k$. Lastly, since $x_k\in B(x_0,r)$ for all $k$, Lemma \ref{fxjbound} also holds for any $j$. This shows that $f(x^*) = 0$ and gives the required bound for~$f(x_j)$. 
\end{proof}

\subsection{Proof of Theorem \ref{deepthm}}\label{deepproof}
Take any $w  = (W_1,b_1,\ldots,W_L, b_L)\in \rr^p$. Define 
\[
f_1(x,w) := \sigma_1(W_1 x + b_1),
\]
and for $2\le \ell \le L$, define 
\[
f_\ell(x,w) := \sigma_\ell(W_\ell\sigma_{\ell-1}(\cdots W_2 \sigma_1(W_1 x+b_1) + b_2\cdots)+b_\ell),
\]
so that $f = f_L$. Note that $f_\ell(\cdot, w)$ is a map from $\rr^d$ into $\rr^{d_\ell}$. Define 
\[
g_1(x,w) := W_1 x+ b_1,
\]
and for $2\le \ell \le L$,  define
\[
g_\ell(x,w) := W_{\ell}f_{\ell-1}(x,w)+b_{\ell},
\]
so that $f_\ell(x,w) = \sigma_{\ell}(g_\ell(x,w))$. Let $D_\ell(x,w)$ be the $d_\ell\times d_{\ell}$ diagonal matrix whose diagonal consists of the elements of the vector $\sigma_\ell'(g_\ell(x,w))$. Let $\partial_{ij} f_\ell$ denote the partial derivative of $f_\ell$ with respect to the $(i,j)^{\mathrm{th}}$ element of $W_1$. Then note that 
\[
\partial_{ij} f_1 = D_1 \partial_{ij} g_1 = D_1 e_ie_j^T x = x_j D_1e_i,
\]
where $e_i\in \rr^{d_1}$ is the vector whose $i^{\mathrm{th}}$ component is $1$ and the rest are zero. Similarly, for $2\le \ell\le L$,
\begin{align*}
\partial_{ij} f_\ell &= D_\ell\partial_{ij} g_\ell = D_\ell W_\ell \partial_{ij} f_{\ell-1}.
\end{align*}
Using this relation and the fact that $D_L=1$, we get
\begin{align}\label{pijfxw}
\partial_{ij} f(x,w) =x_j q_i(x,w),
\end{align}
where
\begin{align}\label{qdef}
q_i(x,w) := W_L D_{L-1}(x,w)W_{L-1} \cdots W_2 D_1(x,w)e_i.
\end{align}
Let $H(w)$ denote the $n\times n$ matrix whose $(i,j)^{\mathrm{th}}$ entry is 
\[
H_{ij}(w) := \nabla_w f(x_i,w)\cdot  \nabla_w f(x_j,w).
\]
Then the definition of $S$ shows that for any $w$ where $S(w)\ne 0$, 
\begin{align}
\frac{|\nabla S(w)|^2}{S(w)} &= \frac{4}{n} \frac{\sum_{i,j=1}^n (y_i-f(x_i,w))(y_j-f(x_j,w)) H_{ij}(w)}{\sum_{i=1}^n (y_i-f(x_i,w))^2}\notag \\
&\ge \frac{4}{n} \lambda(w),\label{sw1}
\end{align}
where $\lambda(w)$ denotes the minimum eigenvalue of $H(w)$. Now note that for any vector $a = (a_1,\ldots,a_n)^T\in \rr^n$ with norm $1$, we have
\begin{align*}
\sum_{i,j=1}^n a_ia_j H_{ij}(w) &= \sum_{i,j=1}^n a_ia_j  \nabla_w f(x_i,w)\cdot  \nabla_w f(x_j,w)\\
&= \biggl|\sum_{i=1}^n a_i \nabla_w f(x_i,w)\biggr|^2= \sum_{j=1}^p \biggl(\sum_{i=1}^n a_i \partial_{w_j} f(x_i,w)\biggr)^2. 
\end{align*} 
We obtain a lower bound on the above term by simply considering those $j$'s that correspond to the entries of $W_1$. By \eqref{pijfxw}, this gives 
\begin{align*}
\sum_{i,j=1}^n a_ia_j H_{ij}(w) &\ge \sum_{r=1}^{d_1}\sum_{s=1}^{d}\biggl(\sum_{i=1}^n a_i \partial_{rs} f(x_i,w)\biggr)^2\\
&=  \sum_{r=1}^{d_1}\sum_{s=1}^{d}\biggl(\sum_{i=1}^n a_i x_{si} q_r(x_i,w)\biggr)^2\ge n\lambda_X \sum_{r=1}^{d_1} \sum_{i=1}^n a_i^2 q_r(x_i,w)^2,
\end{align*}
where $\lambda_X$ is the minimum eigenvalue of $\frac{1}{n}X^TX$, and $X = (x_{si})_{1\le s\le d, 1\le i\le n}$ the $d\times n$ matrix whose $i^{\mathrm{th}}$ column is the vector $x_i$. Since $x_1,\ldots,x_n$ are linearly independent, the Gram matrix $X^TX/n$ is strictly positive definite, which implies that $\lambda_X>0$. Since $|a|=1$, this gives 
\begin{align}\label{sw2}
\lambda(w) &\ge n \lambda_X \sum_{r=1}^{d_1} \min_{1\le i\le n} q_r(x_i,w)^2. 
\end{align}
Now take any $w$ as in the statement of Theorem \ref{deepthm}, that is, 
\begin{itemize}
\item the entries of $W_2,\ldots,W_{L-1}$ are all strictly positive, and 
\item $b_1,\ldots, b_L$ and $W_1$ are zero.
\end{itemize}
Then, since $\sigma_1(0)=0$, we have that $f_1(x,w)=0$ for each $x$. By induction, $f_\ell(x,w)=0$ for each $\ell\le L$, and hence $f(x,w)=0$. Thus,
\begin{align}\label{swupper}
S(w) &= \frac{1}{n}\sum_{i=1}^n y_i^2,
\end{align}
irrespective of the values of $W_2,\ldots, W_L$. 
Let $\delta,K,r,K',a_\ell,c_\ell,P,C$ be as in Theorem~\ref{deepthm} and Remark~\ref{rmk:constants}, and take $w'\in B(w,r)$. Each entry of $W_2',\ldots,W_{L-1}'$ then lies in $[r,K']$. Each row of $W_1'$ has Euclidean norm at most $r$, and each bias entry has absolute value at most $r$. Hence, for every input $x_i$, each coordinate of $g_1(x_i,w')$ has absolute value at most
\[
r|x_i|+r\le r(R_X+1)=a_1.
\]
Inductively, for $2\le\ell\le L-1$, the coordinates of $g_\ell(x_i,w')$ have absolute value at most
\[
K'd_{\ell-1}\gamma_{\ell-1}(a_{\ell-1})+r=a_\ell.
\]
Thus each diagonal entry of $D_\ell(x_i,w')$ is at least $c_\ell>0$, for $1\le\ell\le L-1$. These bounds have no dependence on the output weights.

Let $A:=\min_j(W_L)_{1j}\ge A_0>r$. Each entry of $W_L'$ is at least $A-r$, so~\eqref{qdef} gives, for every $1\le s\le d_1$,
\[
q_s(x_i,w')\ge (A-r)r^{L-2}PC.
\]
When $L=2$, this is simply $q_s(x_i,w')\ge(A-r)c_1$, with no intermediate hidden-layer weight matrices. Equations~\eqref{sw1} and~\eqref{sw2} therefore imply
\[
\frac{|\nabla S(w')|^2}{S(w')}
\ge4\lambda_Xd_1(A-r)^2r^{2L-4}(PC)^2
=\underline\alpha
\]
whenever $w'\in B(w,r)$ and $S(w')>0$.
If $Y_2>0$, the choice~\eqref{eq:A0-explicit} gives
\[
\frac{4S(w)}{r^2\alpha(w,r)}
\le\frac{Y_2}{\lambda_Xd_1(A-r)^2r^{2L-2}(PC)^2}<\frac14.
\]
This verifies~\eqref{mainassump}. Theorems~\ref{mainthm2} and~\ref{mainthm1}, with the step-size bound~\eqref{eq:eta-initialization} for the chosen $w$, prove the two conclusions. If $Y_2=0$, then $S(w)=0$ and both dynamics are stationary. This completes the proof of Theorem~\ref{deepthm}.

\subsection{Proof of Theorem \ref{newdeepthm}}\label{newproof}
Use the constants $a,b,B,\beta,M$ from the theorem, and suppose that $E_M$ occurs. For every $w'\in B(w_0,M)$, all entries of $W_2',\ldots,W_L'$ are at least $M$. The activation derivatives are globally bounded below by $C_1$, so~\eqref{qdef} gives
\[
q_s(x_i,w')\ge C_1^{L-1}M^{L-1}P
\quad\text{for all }i,s.
\]
By~\eqref{sw1} and~\eqref{sw2}, the local PL constant for $S$ satisfies
\begin{equation}\label{eq:random-PL}
\alpha(w_0,M)\ge aM^{2L-2}.
\end{equation}
At initialization, all biases are zero. Since $|\sigma_\ell(t)|\le C_2|t|$ for each hidden-layer activation, while the output activation is the identity,
\begin{align*}
|f(x,w_0)|
&\le C_2^{L-1}\left(\prod_{\ell=2}^L\|W_\ell\|_F\right)|W_1x|\\
&\le 3^{L-1}C_2^{L-1}M^{L-1}
       \left(\prod_{\ell=2}^L\sqrt{d_\ell d_{\ell-1}}\right)|W_1x|\\
&=bM^{L-1}|W_1x|.
\end{align*}
Here $\prod_{\ell=2}^L\sqrt{d_\ell d_{\ell-1}}=\sqrt{d_1}P$, since $d_L=1$. Consequently,
\begin{align}\label{eq:random-loss}
S(w_0)&\le 2Y_2+\frac{2b^2M^{2L-2}}{n}\sum_{i=1}^n|W_1x_i|^2\notag\\
&=2Y_2+\frac{2b^2M^{2L-2}}{n}\tr(W_1XX^TW_1^T)\notag\\
&\le2Y_2+B M^{2L-2}\|W_1\|_F^2.
\end{align}
On $F_M$, the last term is at most $(a/16)M^{2L}$, because $B\beta^2=a/16$. Also,~\eqref{eq:random-M} gives $2Y_2\le(a/16)M^{2L}$. Thus
\begin{equation}\label{eq:random-trapping}
4S(w_0)\le\frac{a}{2}M^{2L}
\le\frac12 M^2\alpha(w_0,M)
<M^2\alpha(w_0,M).
\end{equation}
In particular, the first-layer squared-norm threshold is explicitly
\[
\beta^2M^2=\frac{aM^2}{16B}.
\]
Theorem~\ref{mainthm2} now proves the gradient-flow assertion.

For completeness, we make the uniform step-size choice precise. Let $\mathcal K_M$ be the set of all initializations in $E_M\cap F_M$, with zero biases, and let
\[
\mathcal U_M:=\{w:\operatorname{dist}(w,\mathcal K_M)\le2M\}.
\]
Both sets are compact. As $S$ is $C^2$, the constants
\begin{align*}
\widehat L_1&:=\max\left\{1,\sup_{w\in\mathcal U_M}\max_j|\partial_jS(w)|\right\},\\
\widehat L_2&:=\max\left\{1,\sup_{w\in\mathcal U_M}\max_{j,k}|\partial_j\partial_kS(w)|\right\}
\end{align*}
are finite. For every $w_0\in\mathcal K_M$ with $S(w_0)>0$,~\eqref{eq:random-trapping} gives $\rho\le1/\sqrt2<3/4$, so $\ep=1/4$ is admissible in Theorem~\ref{mainthm1}. Therefore the single bound
\begin{equation}\label{eq:random-eta}
\eta_0:=\min\left\{\frac{M}{\widehat L_1\sqrt p},\,
                       \frac{1}{2\widehat L_2p}\right\}>0
\end{equation}
works for all $w_0\in\mathcal K_M$. Initializations with $S(w_0)=0$ are stationary. The derivative bounds in~\eqref{eq:random-eta} may depend on the full data and on the activation functions, including their second derivatives on this compact set; no dimension-independent step-size bound is claimed.

Finally, $\|W_1\|_F^2$ has distribution $(c/d)\chi^2_{dd_1}$, independently of $E_M$. Each downstream Gaussian entry has a strictly positive probability of lying in $[2M,3M]$, giving exactly the product in Remark~\ref{rmk:random-probability}. In particular, the sufficient event has strictly positive probability. This completes the proof.

\end{document}